\documentclass{hld2024} 

\usepackage[utf8]{inputenc} 
\usepackage[T1]{fontenc}    
\usepackage{hyperref}       
\usepackage{url}            
\usepackage{booktabs}       
\usepackage{nicefrac}       
\usepackage{microtype}      
\usepackage{xcolor}         
\usepackage{amsfonts, amsmath}
\usepackage{algorithm}
\usepackage{algpseudocode}
\usepackage{tikz}
\usepackage{mwe} 
\usepackage{enumitem}
\usepackage{adjustbox}

\newtheorem{defintion}{Definition}
\usetikzlibrary{arrows.meta,bending,chains,decorations.markings,decorations.pathreplacing,positioning}

\newcommand{\resterm}{\mathcal{F}}
\newcommand{\skipterm}{\mathcal{S}}

\newcommand{\skiplayerset}{\mathcal{K}_{\skipterm}}
\newcommand{\inblock}{\mathcal{V}}
\newcommand{\outblock}{\mathcal{H}}
\newcommand{\earlyexit}{\mathcal{E}}
\newcommand{\lossfunc}{\mathcal{L}}
\newcommand{\sample}{x^{(i)}}
\newcommand{\outsample}{\hat{y}^{(i)}}

\title{Towards an Optimal Control Perspective of ResNet Training}

\hldauthor{%
	\Name{Jens Püttschneider} \Email{jens.puettschneider@tu-dortmund.de}\\
	\addr{Institute of Energy Systems, Energy Efficiency and Energy Economics, TU Dortmund, Germany}
	\AND
	\Name{Simon Heilig} \Email{simon.heilig@rub.de}
	\\\addr{Department of Computer Science, Ruhr University Bochum, Germany}
	\AND
	\Name{Asja Fischer} \Email{asja.fischer@rub.de}\\
	\addr{Department of Computer Science, Ruhr University Bochum, Germany}
	\AND
	\Name{Timm Faulwasser} \Email{timm.faulwasser@ieee.org}\\
	\addr{Institute of Control Systems, Hamburg University of Technology, Germany}
}

\begin{document}

\maketitle

\begin{abstract}
We propose a training formulation for ResNets reflecting an optimal control problem that is applicable for standard architectures and general loss functions.
We suggest bridging both worlds via penalizing intermediate outputs of hidden states corresponding to stage cost terms in optimal control.
For standard ResNets, we obtain intermediate outputs by propagating the state through the subsequent skip connections and the output layer.
We demonstrate that our training dynamic biases the weights of the unnecessary deeper residual layers to vanish. This indicates the potential for a theory-grounded layer pruning strategy.
\end{abstract}

\section{Introduction}
Analyzing properties of neural networks and their training dynamics is one of the promising avenues towards trustworthy applications of modern AI. 
In this work, we take the optimal control perspective of neural network training in which the forward propagation of the data over the layers of the network is considered as the state trajectory of a dynamical system \cite{esteve2021largetime, ESTEVEYAGUE2023105452, FAULWASSER2024100290, puttschneider2024dissipativity}. 
The trainable parameters act as control signals that optimally steer the state towards a target representation.
This way, the training emulates the numerical solution to an optimal control problem, resulting in parameters that minimize the loss function evaluated for the terminal state.
The systems and control perspective is particularly insightful for ResNets \cite{10.1007/978-3-319-46493-0_38}, since they represent Euler forward discretizations of continuous-time neural ODEs \cite{chang2018reversible}.
Recent works on optimal control training of ResNets suggest augmenting the training objective by the loss function evaluated for intermediate outputs placed at the states of the hidden residual blocks \cite{esteve2021largetime,  pmlr-v235-miao24a, puttschneider2024dissipativity}. 
From an optimal control perspective, this corresponds to a stage cost in the training objective \footnote{The stage cost loss from optimal control is similar to early exiting, where intermediate predictions are obtained from hidden states. 
In joint training formulations, similar to optimal control, the loss functions corresponding to  the intermediate predictions are included in the training objective \cite{pmlr-v38-lee15a, 10.1145/3698767, 10.1145/3527155}.}.
Placed within the interpolation regime, the training formulation with stage cost gives rise to a self-regularizing ResNet possessing asymptotic behaviors, in which, with sufficient depth, the weights decay to zero and the hidden states approach their target state \cite{esteve2021largetime, ESTEVEYAGUE2023105452, FAULWASSER2024100290, puttschneider2024dissipativity}. 
However, the dynamic systems perspective on ResNets only addresses architectures with a fixed latent dimension. While the work \cite{esteve2021largetime} investigates varying latent dimensions via a continuous space-time neural network, an optimal control perspective to real-world architectures remains an open problem.
Moreover, existing results on stage cost formulations for ResNet training place somewhat restrictive assumptions on the loss function, thereby excluding the most commonly used one -- the cross-entropy with binary targets.

In this paper, we show how this optimal control formulation with stage cost can be extended to more general architectures, including standard ResNets \cite{10.1007/978-3-319-46493-0_38} and generic loss functions. 
For ResNet architectures with 1x1 convolutions in the skip connections, incorporating skip connections in the intermediate outputs and including their respective losses as a stage cost recovers a training scheme following the solution of an optimal control problem. 
By introducing the notion of shallower \textit{SubResNets}, we prove that the performance of the intermediate outputs of the stage cost ResNet is asymptotically bounded by the performance of the shallower SubResNet trained using standard training.
Our asymptotic results show that residual blocks tend to learn identity mappings after forward propagation sufficiently accomplishes the learning task, hence self-regularizing the depth of the network. 
Empirically, we find that our ResNet trained with stage cost converges much faster in terms of depth to near near optimal values of the loss function.  
In short, our paper makes the following contributions:
\begin{itemize}[topsep=0.2em, partopsep=0pt]\itemsep-0.2em
    \item We propose a training scheme for residual-connection based neural networks derived from an optimal control perspective, including stage cost evaluation of loss functions.
    \item We derive asymptotic results on the loss dynamics, pointing towards self-regularizing training in terms of optimal depth.
    \item We empirically assess the convergence of the loss and quantify our loss bounds on standard benchmarks.
\end{itemize}

\section{ResNet Notation}
\begin{figure}
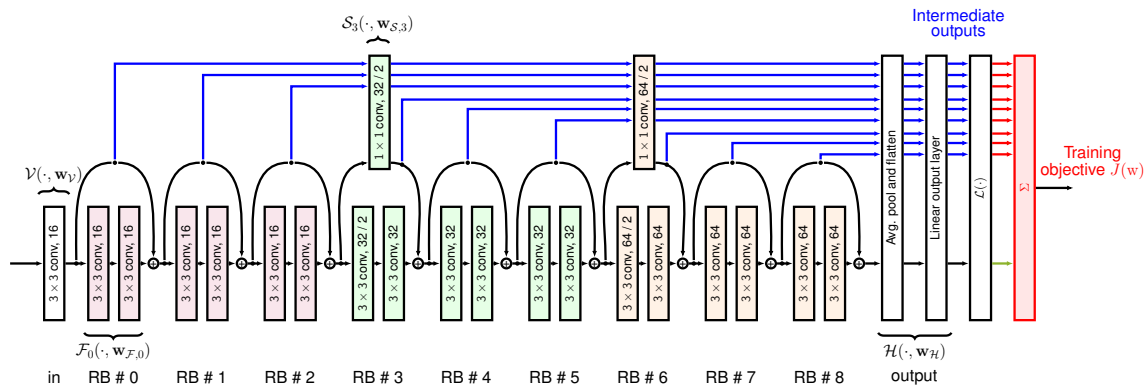

	\newif\ifincludestagecost
	\includestagecosttrue
	\include{figure/resnet}
    \vspace*{-1.2cm}
	\caption{Illustration of a ResNet-v2 architecture while training with our proposed stage cost regularizer penalizing intermediate outputs. We propose to forward the hidden states through the network's non-trivial skip connections to account for dimensionality projections.}
	\vspace{-0.8cm}
	\label{fig:resnet_cifar}
\end{figure}
We briefly revisit the core components of the ResNet-v2 \citep{10.1007/978-3-319-46493-0_38,he2016deep} architecture which is the canvas for this work.

As illustrated in Figure \ref{fig:resnet_cifar}, the forward pass goes as follows: a sample $\sample$ is initially embedded into a latent space via $\inblock(\sample, \mathbf{w}_\inblock)$, then $N$ residual blocks process the input to eventually forward the final representation $\sample_N$ to the output head $\outblock(\sample_N, \mathbf{w}_\outblock)$.
We denote trainable weights as $\mathbf{w} = \left[ \mathbf{w}_{\inblock},\mathbf{w}_{0},...,\mathbf{w}_{N-1}, \mathbf{w}_\outblock \right]$, where $\mathbf{w}_k = \left[\mathbf{w}_{\resterm,k},\mathbf{w}_{\skipterm,k}\right]$ corresponds to the $k$th residual block $f_k(\sample_k, \mathbf{w}_{k})$ that is composed of the convolutional layers $\mathcal{F}_k(\sample_k, \mathbf{w}_{F,k})$ and the skip connection $\skipterm_k (\sample_k, \mathbf{w}_{\skipterm,k})$ with optional parameters $\mathbf{w}_{\skipterm,k}$. If two consecutive blocks have the same dimensionality, an identity skip connection is employed, but if there is a change in dimensionality, a 1x1 convolution is used as a projection layer. For the ease of reading, we denote $\skiplayerset$ as the index set of blocks having non-trivial skip connections.
In summary, a ResNet-$N$ with $N$ residual blocks comprises the following system:
\begin{equation}
	\label{eq:simple_resnet_dynamics}
	\begin{aligned}
 		\sample_0 &= \inblock(\sample, \mathbf{w}_\inblock) \enspace, \\	 
		\sample_{k+1} 
		&= f_k(\sample_k, \mathbf{w}_k) =  \begin{cases}
		\sample_k &+ \resterm_k(\sample_k, \mathbf{w}_{\resterm,k}), \; k \notin \skiplayerset \\
		\skipterm_k (\sample_k, \mathbf{w}_{\skipterm,k}) &+ \resterm_k(\sample_k, \mathbf{w}_{\resterm,k}), \; k \in \skiplayerset, 
		\end{cases}
		\; &\forall k = 0, ..., N&-1 \enspace,\\
		\outsample &= \outblock(\sample_N, \mathbf{w}_\outblock) \enspace.
	\end{aligned}
\end{equation}

A batch of $D$ samples $\mathbf{x} = \left[x^{(1)}, x^{(2)}, \ldots, x^{(D)}\right]$ and its propagation through the network is denoted by boldface variables. Assuming a constant state dimension in the hidden residual blocks, the propagation of the data in the residual blocks can be understood as a time-varying dynamical system: $\mathbf{x}_{k+1} = f_k(\mathbf{x}_k, \mathbf{w}_k) = \mathbf{x}_k + \resterm_k(\mathbf{x}_k, \mathbf{w}_{\resterm,k})$, with stacked data $\mathbf{x}_k$ as state variable and the trainable parameters $\mathbf{w}_k$ as control inputs.
The objective function of standard neural network training is the loss 
$\lossfunc$ at the network output, which corresponds to the terminal penalty in optimal control. 
In addition to common weight decay, optimal control approaches to ResNet training have suggested including a stage cost regularizer $\ell(\mathbf{x}_k, \mathbf{w}_k)$ that depends also on the states $\mathbf{x}_k$ \cite{esteve2021largetime, FAULWASSER2024100290, pmlr-v235-miao24a, puttschneider2024dissipativity}. This leads to the training objective 
$J(\mathbf{w})=  \sum_{k = 0}^{N-1}  \ell(\mathbf{x}_k, \mathbf{w}_k) + \lossfunc (\hat{\mathbf{y}})$.

\section{Optimal Control Perspective on Standard ResNets}
\label{sec:oc_perspective_resnets}
We now extend the optimal control formulation of ResNet training to standard ResNets with non-identity skip connections \cite{10.1007/978-3-319-46493-0_38} that allows for a theoretical analysis of the training. To invoke the training objective using standard ResNet \eqref{eq:simple_resnet_dynamics}, we need a tractable notion of the stage costs $\ell(\mathbf{x}_k, \mathbf{w}_k)$. To this end, we design the intermediate outputs to reuse the existing loss function based on the following observation. 

Consider a ResNet architecture \eqref{eq:simple_resnet_dynamics} with $N$ residual blocks, in which the parameters of the residual branches for $M \leq k \leq N-1$ are set to zero $\mathbf{w}_{\resterm,k}=0$. 
Then, the hidden state $\sample_M$ is forwarded only by the skip connections of the following residual blocks, resulting in the output
$
	\outsample = \outblock(\skipterm_{N-1}(\cdots \skipterm_{M}(\sample_M, \mathbf{w}_{\skipterm,M}), \cdots,  \mathbf{w}_{\skipterm,N-1}), \mathbf{w}_{\outblock})
$.
Moreover, since most residual blocks have identity skip connections, the remaining number of layers the state $\sample_M$ gets propagated is comparatively low.
Hence, we can use those non-trivial skip connections as building blocks for obtaining the intermediate output of state $\sample_M$, since, similar to the standard design of early exits, they should consist of few layers for computational efficiency \cite{10.1145/3698767, 10.1145/3527155}.
This leads us to the definition of intermediate output heads $\earlyexit_k$, which predict $\outsample_k$ based on intermediate states $\sample_k$ via:
\begin{equation}
		\outsample_k 
		= \earlyexit_k(\sample_k, \mathbf{w}_{\earlyexit,k}) 
		= \outblock(\skipterm_{N-1}(\cdots \skipterm_{k}(\sample_k, \mathbf{w}_{\skipterm,k}), \cdots,  \mathbf{w}_{\skipterm,N-1}), \mathbf{w}_{\outblock}) \quad\forall k = 0, \dots, N-1.
	\label{eq:proposed_early_exit}
\end{equation}
Figure \ref{fig:resnet_cifar} illustrates the proposed flow of data through the skip connections of the original blocks.
Now, if after residual block $M < N$ the parameters of the residual terms are set to zero, i.e., $\mathbf{w}_{\resterm, k} = 0$ for all $k = M, ..., N-1$, then all subsequent intermediate outputs and the final output are equal to the intermediate output $\hat{\mathbf{y}}_M$, i.e., 
$\hat{\mathbf{y}}_M$, $\hat{\mathbf{y}} = \hat{\mathbf{y}} _k = \hat{\mathbf{y}}_M $ for all $k = M, ..., N$.
A formal result of this is provided in Lemma \ref{lemma:identity_mapping} in the Appendix.
For a ResNet with only identity skip connections, the identity mapping also holds for the hidden states of the residual blocks, i.e., $\mathbf{w}_{\resterm, k} = 0$ results in $\mathbf{x}_k = \mathbf{x}_{k+1}$, which, from a control perspective, corresponds to an equilibrium state.
The parameters $\mathbf{w}_{\earlyexit,k}$ of the intermediate outputs can be either shared with the backbone network (i.e., $\mathbf{w}_{\earlyexit,k} = \left[ \mathbf{w}_{\skipterm, k}, ..., \mathbf{w}_{\skipterm,N-1}, \mathbf{w}_{\outblock}\right]$) or can be additional parameters for increased representational capacity. This way, we can re-interpret or fine-tune ResNets trained without our intermediate outputs, since all necessary parameters are at hand.

Now, our optimal control problem implements the stage cost as the prediction loss of the intermediate output \eqref{eq:proposed_early_exit} with a weight $\gamma \geq 0$, where $\gamma=0$ corresponds to the standard training,
\begin{equation}	\label{eq:training_problem}
		\min_{\mathbf{w}} \; J_N(\mathbf{w})=  \sum_{k = 0}^{N-1} \gamma \lossfunc (\hat{\mathbf{y}}_k) + \lossfunc (\hat{\mathbf{y}}) \quad  \text{subject to} \quad \text{\eqref{eq:simple_resnet_dynamics}, \eqref{eq:proposed_early_exit}},  \; \forall k = 0, ..., N-1.
\end{equation}
Next, we use this setup to derive bounds on the loss behavior of a deep ResNet trained with the stage cost loss by relating it to the loss of a shallower network sharing the architecture.
To this end, we define a ResNet-$M$ with $M$ residual blocks and $M < N$ as a SubResNet architecture of a deeper ResNet-$N$ with $N$ residual blocks if the output of ResNet-$M$ corresponds to the same architecture as the $M$-th intermediate output of ResNet-$N$.
A formal definition is provided in Definition~\ref{def:subresnet} in Appendix~\ref{sec:depth_bound}.
This implies that if for a ResNet-$N$ and its SubResNet ResNet-$M$ the shared parameters are equal, then the intermediate output of ResNet-$N$ after $M$ residual blocks $\hat{\mathbf{y}}^N_M$ will equal the output $\hat{\mathbf{y}}^M$ of the shallower ResNet-$M$, i.e., $\hat{\mathbf{y}}^N_M = \hat{\mathbf{y}}^M$.
Moreover, if the additional parameters of ResNet-$N$, namely the parameters of its residual blocks $M, \ldots, N-1$, are set to zero, i.e., $\mathbf{w}^N_{\resterm, k} = 0$ for $k = M, \ldots, N{-}1$, then all subsequent intermediate outputs of the deeper ResNet-N will produce the same prediction $\hat{\mathbf{y}}^{N}_k = \hat{\mathbf{y}}^{N} = \hat{\mathbf{y}}^{M}$, for all $k = M, \ldots, N$.
This relationship is formalized in Lemma~\ref{lemma:identity_mapping_subresnets} in Appendix~\ref{sec:depth_bound}.
Building upon the identity mapping property of the ResNet, we now derive the following result on the asymptotic loss behavior of a ResNet trained with stage cost, based on the performance of a shallower SubResNet with traditional training.
\begin{theorem}[ResNet asymptotics with weight decay]
	\label{thrm:deep_resnet_asymptotics_wd}
	Consider a ResNet-$N$ with $N$ residual blocks trained by \eqref{eq:training_problem} with stage cost weight $\gamma > 0$, 
	and a weight decay $\dfrac{\lambda }{2} \lVert\mathbf{w}^N_{\resterm,k}\rVert^2$ on the parameters on the residual terms $\mathbf{w}^N_{\resterm,k}$.
	Let ResNet-M be its SubResNet (Definition \ref{def:subresnet}), trained without the stage cost but the same weight decay.
	Let $\bar{\mathcal{L}}=\lossfunc(\hat{\mathbf{y}}^M)$ be the loss function of the output of ResNet-M.
	Then, the training objective for the optimal parameters $\mathbf{w}^N$ of ResNet-$N$ satisfies the upper bound
	\begin{equation*}
		J_N(\mathbf{w}^N) = \sum_{k = 0}^{N-1} \left[ \gamma \lossfunc(\hat{\mathbf{y}}^N_k)  +
		\frac{\lambda}{2} \lVert\mathbf{w}^N_{\resterm,k}\rVert^2 \right] + 
		\lossfunc(\hat{\mathbf{y}}^N) 
		\leq 
		\bar{J}_N,
	\end{equation*}
	with $\bar{J}_N = \sum_{k = 0}^{M-1} \left[\gamma \lossfunc(\hat{\mathbf{y}}^M_k) + \frac{\lambda}{2} \lVert\mathbf{w}^M_{\resterm,k}\rVert^2 \right]
	+ (1 + \gamma(N - M)) \bar{\mathcal{L}}$ based on ResNet-$M$.
\end{theorem}
The proof is provided in Appendix~\ref{sec:depth_bound}.
If no weight decay is used, an adapted version of this bound yields $\mathcal{L}_\mathrm{avg}= \frac{1}{N+1} \sum_{k = 0}^{N} \lossfunc(\hat{\mathbf{y}}^N_k) \leq \bar{ \mathcal{L}}_\mathrm{avg} = \bar{\mathcal{L}} + \frac{C}{N+1}$, with $C< \infty$ based on the SubResNet-$M$.
Moreover, if a loss function that attains its minimum at zero is used, e.g., the L2 loss, and assuming that there exists a SubResNet of depth $N$ such that $\bar{\mathcal{L}} =0$, then this bound tightens to $\mathcal{L}_\mathrm{avg} \leq \frac{C}{N+1}$.
This result is formalized in Theorem~\ref{lemma:deep_resnet_asymptotics} in Appendix~\ref{sec:depth_bound}.

\section{Experiments}\label{sec:experiments}
In our experiments, we train ResNets on MNIST \cite{lecun1998}, CIFAR-10, and CIFAR-100 \cite{krizhevsky2009learning} using the proposed stage cost loss and compare their loss and accuracy trajectories to those of ResNets of varying depth trained with standard training.
The results for CIFAR-10 are shown in Figure \ref{fig:cifar10_standard} and the implementation details are provided in Appendix \ref{sec:experimental_setup}.
\paragraph{Results.}
Our empirical evaluation shows that the bounds provided in \ref{thrm:deep_resnet_asymptotics_wd} and \ref{lemma:deep_resnet_asymptotics} are relatively tight,
see Appendix \ref{sec:depth_bound}.
The loss trajectories shown in Figure \hyperref[fig:loss_cifar10_standard]{2 (b)} show that ResNets trained with the proposed stage cost already achieve a good fit after 12 residual blocks, reaching 78.74\% test accuracy for the ResNet\textsubscript{WS} in which the parameters of the intermediate outputs are shared.
Performance plateaus for the rest of the first stage of layers with the same number of filters and improves at the onset of subsequent stages with increasing representational capacity.
In contrast, networks trained with standard training only show a good fit at their final output, but with higher test accuracy, especially in the first two stages.
Appendix~\ref{sec:additional_exp} provides an in-depth analysis of intermediate output convergence and parameter dynamics, along with results from training a homogeneous ResNet with a fixed number of filters.
Interestingly, our stage cost loss formulation leads to a principled way of identifying layers that could be pruned due to insignificant residual contribution.
Table~\ref{tab:test_accuracies} compares test accuracies of pruned SubResNet-12 to standard training, showing pruned homogeneous models differ only up to 3.5\% points, while it remains challenging to prune standard ResNets nearly lossless. 
\begin{figure}[t]
	\centering
	\subfigure[Loss.]{%
		\label{fig:loss_cifar10_standard}
		\includegraphics[trim=0 0 225 0, clip, height=3.8cm]{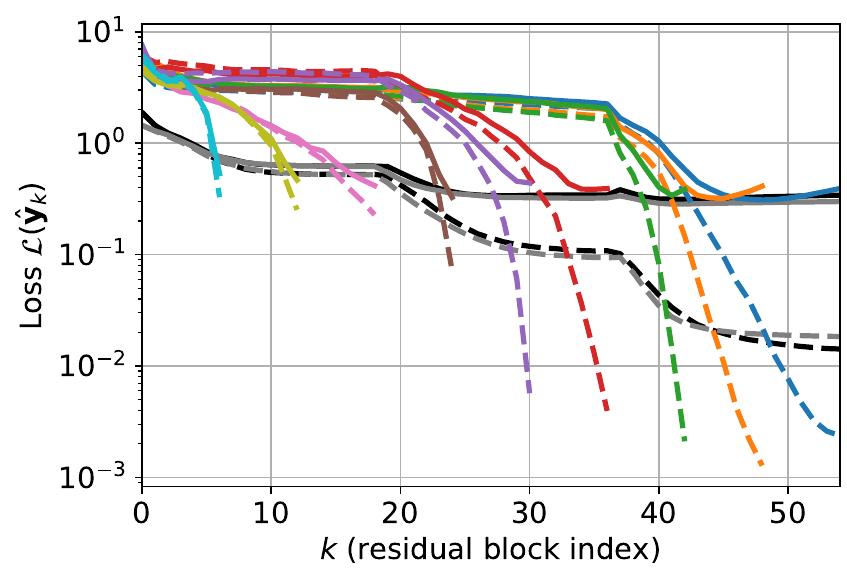}
	}
	\subfigure[Test accuracy.\hspace{2cm}]{%
		\includegraphics[trim=0 0 0 0, clip, height=3.8cm]{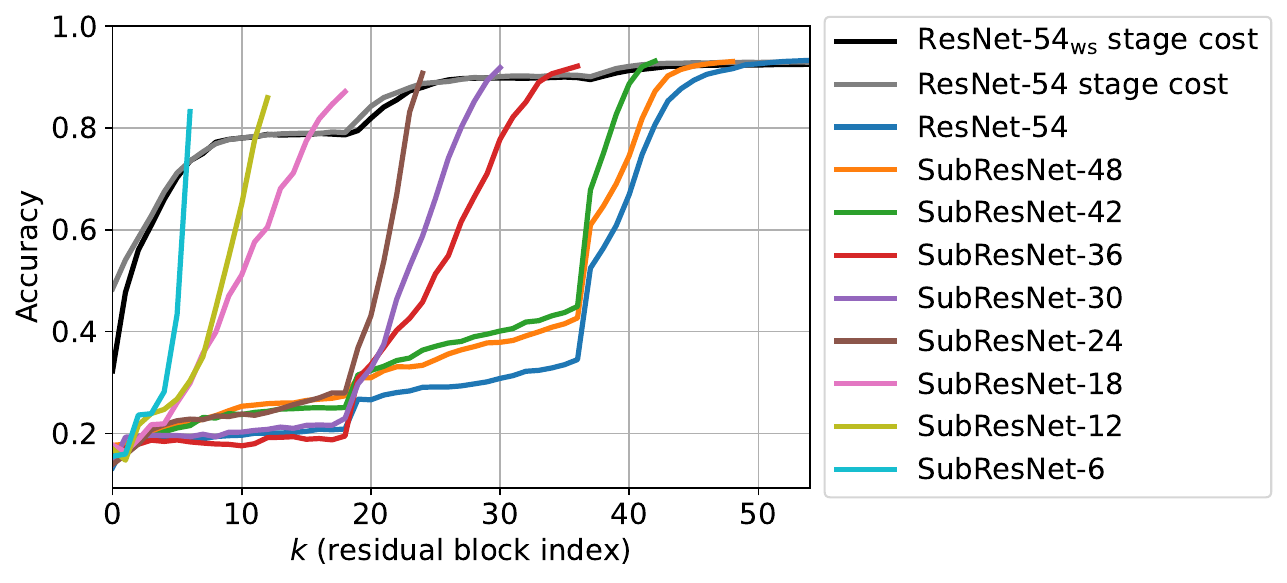}\label{fig:accuracy_cifar10_standard}
	}
	\vspace*{-0.3cm}
	\caption{Loss and accuracy trajectories for the standard ResNets trained on CIFAR-10 with the stage cost loss and their SubResNets of varying depth. The dashed loss lines are the training losses, and the solid lines are the test losses.}
	\label{fig:cifar10_standard}
    \vspace*{-0.5cm}
\end{figure}
\begin{table}[t]
	\centering
	\caption{Test accuracies (\%) and the number of parameters (in millions M) for ResNets trained with the stage cost loss and standard training.}	\label{tab:test_accuracies}
	\vspace*{-0.3cm}	
    \begin{adjustbox}{max width=0.95\textwidth}
	\begin{tabular}{lcccc}
		\toprule
		& \multicolumn{3}{c}{\textbf{Homogeneous ResNet}} & \multicolumn{1}{c}{\textbf{Standard ResNet}} \\
		\cmidrule(lr){2-4} \cmidrule(lr){5-5}
		\textbf{Model} & \textbf{MNIST} & \textbf{CIFAR-10} & \textbf{CIFAR-100} & \textbf{CIFAR-10} \\
		\midrule
		ResNet-54 standard training     					& 99.64 (1.0 M) 	& 93.05 (4.0 M) & 71.94 (4.0 M) & 93.28 (1.7 M) \\
		SubResNet-12 standard training        				& 99.63 (0.2 M)	& 91.43 (0.9 M) & 68.77 (0.9 M) & 85.93 (0.06 M) \\
		ResNet-54$_{\text{ws}}$ stage cost 					& 99.62 (1.0 M) 	& 91.59 (4.0 M) & 69.97 (4.0 M) & 92.51 (1.7 M) \\
		$\quad$ SubResNet-12 pruned							& 99.54 (0.2 M)	& 91.02 (0.9 M) & 66.55 (0.9 M) & 78.74  (0.06 M) \\
		ResNet-54 stage cost								& 99.68 (1.0 M) 	& 91.64 (4.0 M) & 69.18 (4.4 M) & 92.91 (1.9 M) \\
		$\quad$ SubResNet-12 pruned							& 99.64 (0.2 M)	& 91.26 (0.9 M) & 66.66 (0.9 M) & 78.64 (0.06 M) \\
		\bottomrule
	\end{tabular}
    \end{adjustbox}
\end{table}
\paragraph{Outlook.}
Our proposed training scheme offers promising future directions. The post-hoc pruning can be understood as finding the optimal SubResNet, computing the best representations for a downstream task. When considering varying dimensionality across residual blocks, our current analysis does not lead to straightforward pruning like in the simpler homogeneous models. In the future, we want to investigate the impact of the observed stage-wise dynamics, possibly leading us to tighter bounds and eventually surgical pruning of intermediate layers.

\section*{Acknowledgments}
This work has been partially funded by the German Federal Ministry of Research, Technology, and Space (BMFTR) via the 6GEM research hub (16KISK038), and by Deutsche Forschungsgemeinschaft (DFG, German Research Foundation) via the TRR 391 Spatio-temporal Statistics for the Transition of Energy and Transport (project 520388526) and the research unit Active Learning for Systems and Control (project 535860958).

\bibliography{sources}

\begin{thebibliography}{13}
\providecommand{\natexlab}[1]{#1}
\providecommand{\url}[1]{\texttt{#1}}
\expandafter\ifx\csname urlstyle\endcsname\relax
  \providecommand{\doi}[1]{doi: #1}\else
  \providecommand{\doi}{doi: \begingroup \urlstyle{rm}\Url}\fi

\bibitem[Chang et~al.(2018)Chang, Meng, Haber, Ruthotto, Begert, and
  Holtham]{chang2018reversible}
Bo~Chang, Lili Meng, Eldad Haber, Lars Ruthotto, David Begert, and Elliot
  Holtham.
\newblock Reversible architectures for arbitrarily deep residual neural
  networks.
\newblock In \emph{The 32th Annual AAAI Conference on Artificial Antelligence},
  2018.

\bibitem[Esteve and Geshkovski(2023)]{ESTEVEYAGUE2023105452}
Carlos Esteve and Borjan Geshkovski.
\newblock Sparsity in long-time control of neural odes.
\newblock \emph{Systems \& Control Letters}, 172:\penalty0 105452, 2023.

\bibitem[Esteve et~al.(2020)Esteve, Geshkovski, Pighin, and
  Zuazua]{esteve2021largetime}
Carlos Esteve, Borjan Geshkovski, Dario Pighin, and Enrique Zuazua.
\newblock Large-time asymptotics in deep learning.
\newblock \emph{arXiv preprint arXiv:2008.02491}, 2020.

\bibitem[Faulwasser et~al.(2024)Faulwasser, Hempel, and
  Streif]{FAULWASSER2024100290}
Timm Faulwasser, Arne-Jens Hempel, and Stefan Streif.
\newblock On the turnpike to design of deep neural networks: Explicit depth
  bounds.
\newblock \emph{IFAC Journal of Systems and Control}, 30:\penalty0 100290,
  2024.

\bibitem[He et~al.(2016{\natexlab{a}})He, Zhang, Ren, and
  Sun]{10.1007/978-3-319-46493-0_38}
Kaiming He, Xiangyu Zhang, Shaoqing Ren, and Jian Sun.
\newblock Identity mappings in deep residual networks.
\newblock In \emph{European Conference on Computer Vision (ECCV)},
  2016{\natexlab{a}}.

\bibitem[He et~al.(2016{\natexlab{b}})He, Zhang, Ren, and Sun]{he2016deep}
Kaiming He, Xiangyu Zhang, Shaoqing Ren, and Jian Sun.
\newblock Deep residual learning for image recognition.
\newblock In \emph{IEEE Conference on Computer Vision and Pattern Recognition
  (CVPR)}, 2016{\natexlab{b}}.

\bibitem[Krizhevsky(2009)]{krizhevsky2009learning}
Alex Krizhevsky.
\newblock Learning multiple layers of features from tiny images.
\newblock Technical report, 2009.
\newblock Technical Report.

\bibitem[Lecun et~al.(1998)Lecun, Bottou, Bengio, and Haffner]{lecun1998}
Y.~Lecun, L.~Bottou, Y.~Bengio, and P.~Haffner.
\newblock Gradient-based learning applied to document recognition.
\newblock \emph{Proceedings of the IEEE}, 86\penalty0 (11):\penalty0
  2278--2324, 1998.
\newblock \doi{10.1109/5.726791}.

\bibitem[Lee et~al.(2015)Lee, Xie, Gallagher, Zhang, and Tu]{pmlr-v38-lee15a}
Chen-Yu Lee, Saining Xie, Patrick Gallagher, Zhengyou Zhang, and Zhuowen Tu.
\newblock {Deeply-Supervised Nets}.
\newblock In \emph{International Conference on Artificial Intelligence and
  Statistics (AISTATS)}, 2015.

\bibitem[Matsubara et~al.(2022)Matsubara, Levorato, and
  Restuccia]{10.1145/3527155}
Yoshitomo Matsubara, Marco Levorato, and Francesco Restuccia.
\newblock Split computing and early exiting for deep learning applications:
  Survey and research challenges.
\newblock \emph{ACM Computing Surveys}, 55\penalty0 (5), December 2022.

\bibitem[Miao and Gatsis(2024)]{pmlr-v235-miao24a}
Keyan Miao and Konstantinos Gatsis.
\newblock How deep do we need: Accelerating training and inference of neural
  {ODE}s via control perspective.
\newblock In \emph{International Conference on Machine Learning (ICML)}, 2024.

\bibitem[P{\"u}ttschneider and
  Faulwasser(2024)]{puttschneider2024dissipativity}
Jens P{\"u}ttschneider and Timm Faulwasser.
\newblock On dissipativity of cross-entropy loss in training resnets.
\newblock \emph{arXiv preprint arXiv:2405.19013}, 2024.

\bibitem[Rahmath~P et~al.(2024)Rahmath~P, Srivastava, Chaurasia, Pacheco, and
  Couto]{10.1145/3698767}
Haseena Rahmath~P, Vishal Srivastava, Kuldeep Chaurasia, Roberto~G. Pacheco,
  and Rodrigo~S. Couto.
\newblock Early-exit deep neural network - a comprehensive survey.
\newblock \emph{ACM Computing Surveys}, 57\penalty0 (3), November 2024.

\end{thebibliography}

\appendix

\section{Experimental Supplements}
\subsection{Setup}
\label{sec:experimental_setup}
The experimental setup is based on \cite{10.1007/978-3-319-46493-0_38}.
All of our experiments used a pre-activation ResNet with 54 residual blocks, with two convolutional layers per residual block. 
In contrast to our notation, \cite{10.1007/978-3-319-46493-0_38} defines the depth of a ResNet as the total number of layers, including convolutional and fully connected layers. 
Hence, the ResNet-54 in this paper corresponds to the ResNet-110 in \cite{10.1007/978-3-319-46493-0_38}.
The standard ResNet architecture consists of an input embedding with of a 3x3 convolution with 16 filters followed by a batch normalization and ReLU.
This is followed by 54 residual blocks organized in three stages, each with 18 residual blocks.
The convolutional layers within these stages use 16, 32, and 64 filters, respectively.
The output block consists of a batch normalization layer, a ReLU activation, global average pooling, and a final linear layer that predicts the logits.
Similar to \cite{10.1007/978-3-319-46493-0_38}, we train for 165 epochs with a batch size of 128 using SGD with an initial learning rate of $\alpha=0.1$  and a momentum of $\mu = 0.9$, after 82 epochs, the learning rate is divided by ten and again after 123 epochs.
A constant stage cost weight of $\gamma = 0.02$ is used for all experiments.

For ResNets trained without the stage cost, we reuse the backbone network parameters to compute the trajectory of the intermediate outputs.
When training with the stage cost, we consider two variants: ResNet\textsubscript{WS}, which shares parameters between the backbone and the intermediate outputs, and the full stage-cost ResNet, which introduces additional parameters for the intermediate outputs.
For the additional parameters associated with the intermediate outputs, the learning rate is set to $\alpha / \gamma$ to ensure that their effective learning rate matches that of the final output layer.
This latter approach increases the total number of parameters during training.
However, when pruning, these additional intermediate output blocks and their parameters are removed, such that all models have the same number of parameters.
In our training approach, the output block is applied to all hidden states.
For each of these states use a separate batch norm statistics, while sharing the batch norm weights and bias for the ResNet\textsubscript{WS}.

The homogeneous ResNets have a constant number of convolutional filters, $64$  CIFAR and $32$ for MNIST, in all residual blocks and the input embedding.
Apart from this modification, their architecture is identical to the standard ResNet.

\subsection{Additional Experiments}\label{sec:additional_exp}
For the standard ResNet shown in Figure \ref{fig:cifar10_standard}, we observe that accuracy and loss trajectories improve noticeably at the beginning of each stage.
This effect is likely due to the increased representational capacity resulting from the doubling of convolutional filters from one stage to the next.
To verify this, we consider a homogeneous ResNet, in which all residual blocks use a constant number of filters, fixed at 64.
The resulting loss and accuracy trajectories for CIFAR-10 are shown in Figure~\ref{fig:cifar10_homogeneous}.
In this homogeneous architecture, convergence is stronger compared to the standard ResNet. 
After approximately 12 residual blocks, the loss and accuracy trajectories stabilize, and no further improvement is observed. 
This suggests that the remaining residual blocks can be pruned without significant performance loss. 
The resulting SubResNet-12 achieves a test accuracy of 91.26\%, compared to the 91.64\% of the full ResNet-54, which is close to the 91.43\% obtained by training a SubResNet-12 by standard training (see Table \ref{tab:test_accuracies}). 
Hence, the proposed training formulation helps to determine the depth suitable for a given task.
The same convergence behavior is observed for the homogeneous ResNets trained on MNIST and CIFAR-100.
To close the performance gap after pruning, we plan to invest future research into fine-tuning for a small number of epochs.

Notably, our ablation using weight-sharing in the intermediate outputs, denoted as ResNet\textsubscript{WS}, revealed that a ResNet with additional parameters only achieved insignificant accuracy gains. This indicates that our stage-cost trained weight-sharing ResNet\textsubscript{WS} can be used to achieve comparable performance while using a reduced number of parameters.
 
Moreover, we can study the convergence in the residuals between consecutive intermediate outputs $\hat{\mathbf{y}}_{k+1} - \hat{\mathbf{y}}_{k}$ and the norm of the parameters are shown in Figure \ref{fig:cifar10_standard_norms} for the standard ResNet and in Figure \ref{fig:cifar10_homogeneous_norms} for the homogeneous ResNet (using 64 filters across layers).
We observe in Figure \hyperref[fig:cifar10_standard_norms]{4 (b)} and Figure \hyperref[fig:cifar10_homogeneous_norms]{5 (b)} that while the parameter norms converge for all ResNet variants, the parameters of the ResNet-54 trained with the stage cost converge faster than those of the same model trained without it. Especially in the homogeneous case, in Figure \hyperref[fig:cifar10_homogeneous_norms]{5 (a)}, the residual contribution of subsequent blocks vanishes when training with our stage cost loss. This underscores the ability to discard them from forward propagation entirely. When considering the standard setup in Figure \hyperref[fig:cifar10_standard_norms]{4 (a)}, this implication holds only for stages operating in the same representational space, in particular for the last stage. We plan to conduct more research on this effect to eventually be able to surgically prune blocks.

\begin{figure}[H]
	\centering
	
	\subfigure[Loss.]{%
		\includegraphics[trim=0 0 225 0, clip, height=3.9cm]{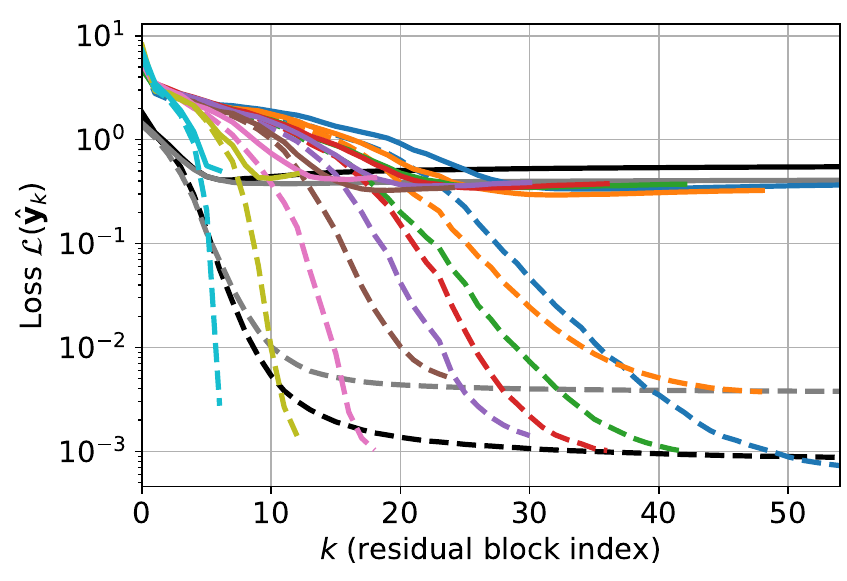}
		\label{fig:loss_cifar10_homogeneous}
	}
	\subfigure[Test accuracy.\hspace{2cm}]{%
		\includegraphics[trim=0 0 0 0, clip, height=3.9cm]{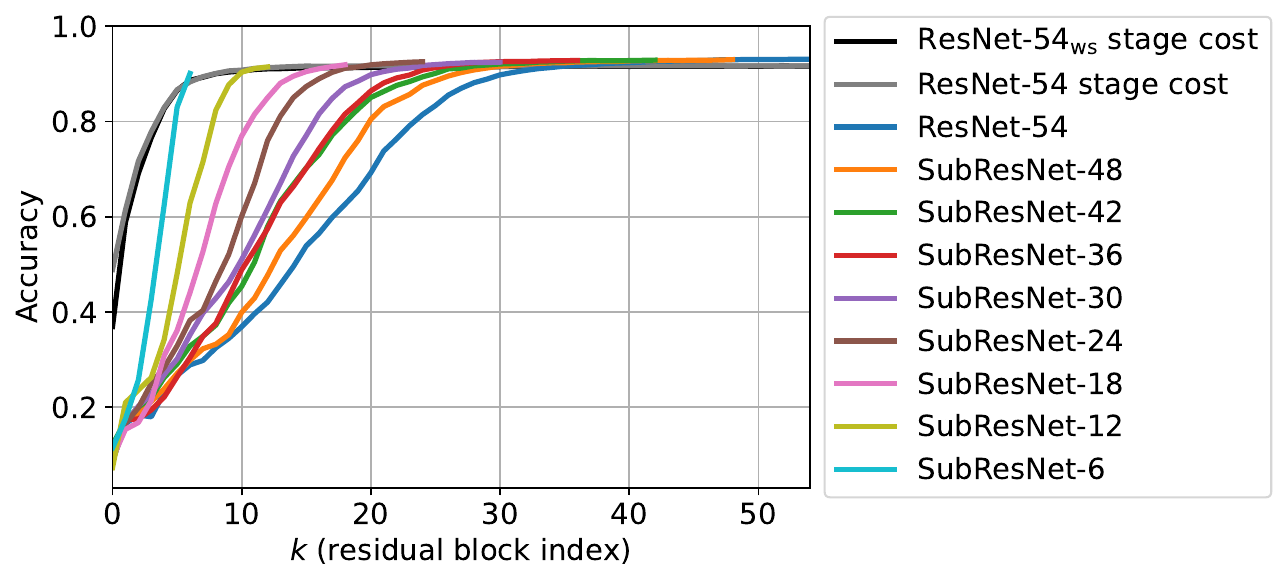}
		\label{fig:accuracy_cifar10_homogeneous}
	}
	\caption{Loss and accuracy trajectories for the ResNets with homogeneous architecture trained on CIFAR-10 with the stage cost loss and their SubResNets of varying depth. The dashed loss lines represent are the training losses and the solid lines are the test losses.}
	\label{fig:cifar10_homogeneous}
\end{figure}

\begin{figure}[H]
	\centering
	 \subfigure[Difference norms]{%
		\includegraphics[trim=0 0 225 0, clip, height=3.9cm]{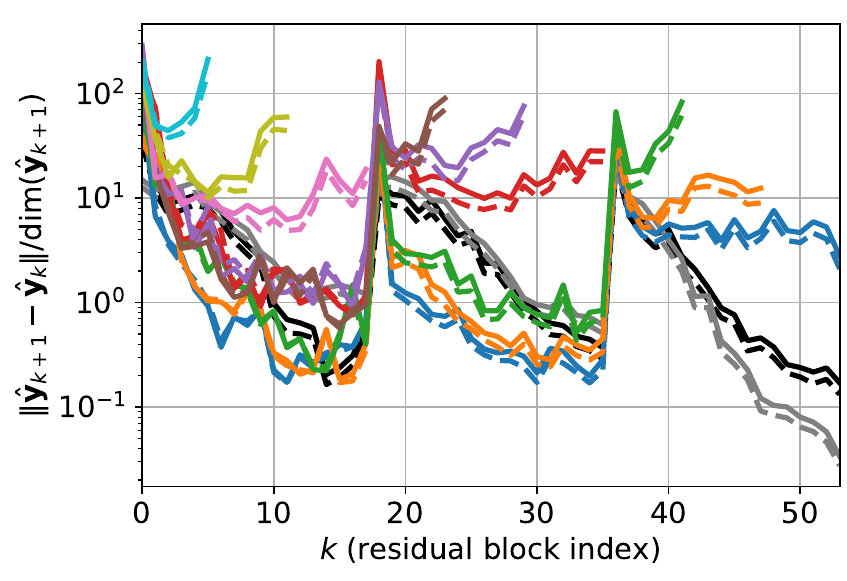}
	}
	\hfill
	\subfigure[Parameter norms.\hspace{2cm}]{%
		\includegraphics[height=3.9cm]{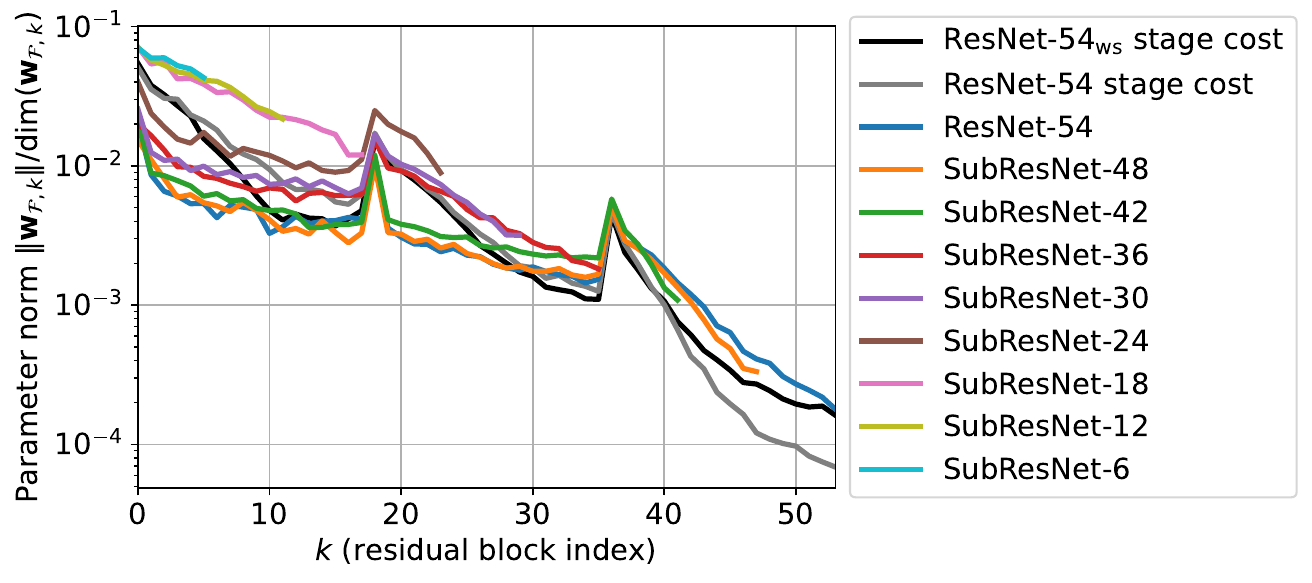}
	}
	\caption{Output residual norms and parameter norm residual trajectories for the standard ResNets trained on CIFAR-10 with the stage cost loss and their SubResNets of varying depth. The dashed loss lines represent are the training losses and the solid lines are the test losses.}
	\label{fig:cifar10_standard_norms}
\end{figure}

\begin{figure}[H]
	\centering
	\subfigure[Difference norms]{%
		\includegraphics[trim=0 0 225 0, clip, height=3.9cm]{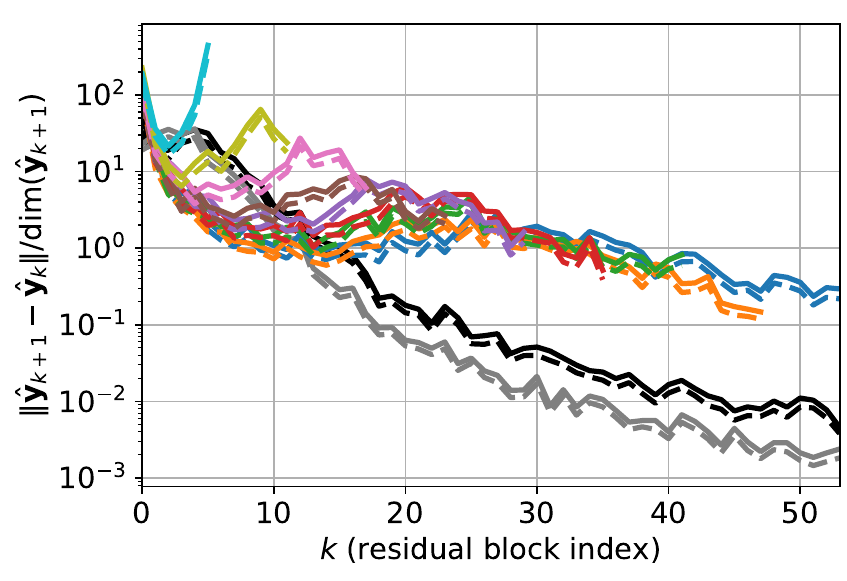}
	}
	\hfill
	\subfigure[Parameter norms.\hspace{2cm}]{%
		\includegraphics[height=3.9cm]{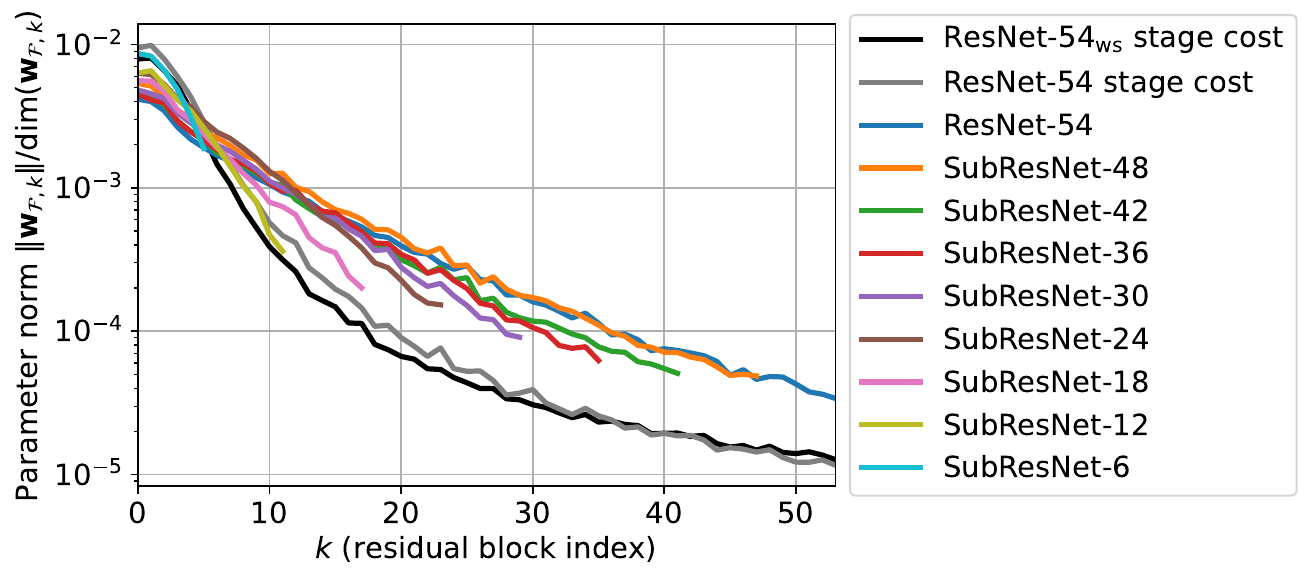}
	}
	\caption{Output residual norms and parameter norm residual trajectories for the homogeneous ResNets trained on CIFAR-10 with the stage cost loss and their SubResNets of varying depth. The dashed loss lines represent are the training losses and the solid lines are the test losses. }
	\label{fig:cifar10_homogeneous_norms}
\end{figure}

\section{Asymptotic Loss Bounds}
\label{sec:depth_bound}
In this appendix, we provide the formal results corresponding to Section \ref{sec:oc_perspective_resnets}. To prove Theorem \ref{thrm:deep_resnet_asymptotics_wd}, we build a rigorous setup with intermediate results. The main idea is based on the observation that with our proposed intermediate outputs \eqref{eq:proposed_early_exit}, the ResNet can learn an identity mapping between subsequent early exits.
\begin{lemma}[Identity mapping]
	\label{lemma:identity_mapping}
	Consider the ResNet \eqref{eq:simple_resnet_dynamics}, with the intermediate outputs chosen according to \eqref{eq:proposed_early_exit}. 
	Let the intermediate outputs of residual block $M$ produce the prediction $\hat{\mathbf{y}}_M$, and let the parameters of all following residual terms be zero, $\mathbf{w}_{\resterm, k} = 0$, $k = M, ..., N-1$. Then all subsequent intermediate outputs and the final output are equal
	\begin{equation}
		\label{eq:identity_mapping}
		\hat{\mathbf{y}} = \hat{\mathbf{y}} _k = \hat{\mathbf{y}}_M,
	\end{equation}
	for all $k = M, ..., N$.
\end{lemma}
\begin{proof}
	If $\mathbf{w}_{\resterm,k}=0$, for the residual blocks $k= M, ..., N-1$, then $\resterm_k(\mathbf{x}_k, 0) = 0$ for each of these blocks.
	Hence, the state is being propagated by $\mathbf{x}_{k+1} = \skipterm_k(\mathbf{x}_k, \mathbf{w}_{\skipterm,k})$ for all $k= M, ..., N-1$. 
	This gives the output $\hat{\mathbf{y}} = \outblock(\mathbf{x}_N, \mathbf{w}_\outblock)$ and the intermediate outputs \eqref{eq:proposed_early_exit}, which corresponds to \eqref{eq:identity_mapping}.
\end{proof}

\begin{defintion}[SubResNet architecture]
	\label{def:subresnet}
	We refer to the architecture of ResNet-$M$ with $M$ residual blocks	\footnote{For simplicity, we refer to the network architecture for $M= 0$, which only consists of the input stem and the output layer, as a ResNet, despite its lack of skip connections.}
	as a SubResNet architecture of ResNet-$N$ with $N$ residual blocks,  $N > M$, if the architecture of the output $\outblock^M$ of ResNet-$M$ is shared with the architecture of the intermediate output $\earlyexit^N_M$ of ResNet-$N$.
	That is, if all residual blocks of ResNet-$M$ are contained in ResNet-$N$, i.e. $f^M_k = f^N_k$ for all $k = 0, ..., M-1$.
	Moreover, if ResNet-$N$ has non identity skip connections in residual blocks $k \geq M$, then those need to be included in the output of ResNet-$M$, i.e. $\outblock^M(x, \mathbf{w}_\mathrm{\outblock}^M) =\earlyexit^N_k(x, \mathbf{w}_{\earlyexit, k}^N)$, with 
	$\mathbf{w}_\outblock^M  = \left[ \mathbf{w}^N_{\skipterm,M}, ..., \mathbf{w}^N_{\skipterm,N-1}, \mathbf{w}_{\outblock}^N\right]$.
\end{defintion}
This definition of the SubResNet architecture implies that the identity mapping (Lemma \ref{lemma:identity_mapping}) also holds between a ResNet and its SubResNets.
\begin{lemma}[Identity mapping for SubResNets]
	\label{lemma:identity_mapping_subresnets}
	Consider ResNet-$N$ with $N$ residual blocks and ResNet-$M$ with a SubResNet architecture (Definition \ref{def:subresnet}), consisting of $M$ residual blocks and producing the output $\hat{\mathbf{y}}^{M}$.
	Let the parameters $\bar{\mathbf{w}}^N$ of the first $M$ residual blocks of ResNet-$N$ be equal to those of ResNet-$M$, i.e. $\bar{\mathbf{w}}_k^N = \mathbf{w}_k^M$ $k = 0, ..., M-1$, and choose the parameters of following residual terms equal to zero, $\bar{\mathbf{w}}^N_{\resterm, k} = 0$ for $k = M, ..., N-1$.
	Then, all intermediate outputs deeper than $M$ and the output of ResNet-$N$ are equal to the output of ResNet-$M$, i.e., $\hat{\mathbf{y}}^{N}_k = \hat{\mathbf{y}}^{N}= \hat{\mathbf{y}}^{M}$ for all $k = M, ..., N$.
\end{lemma}
The proof follows directly from Lemma \ref{lemma:identity_mapping} and is thus omitted.
We are now ready to prove our main result, which is repeated for the ease of readability. 
\newtheorem*{T1}{Theorem~1 (ResNet asymptotics with weight decay)}
\begin{T1}
	Consider a ResNet-$N$ with $N$ residual blocks trained by \eqref{eq:training_problem} with stage cost weight $\gamma > 0$, 
	and a weight decay $\dfrac{\lambda }{2} \lVert\mathbf{w}^N_{\resterm,k}\rVert^2$ on the parameters on the residual terms $\mathbf{w}^N_{\resterm,k}$.
	Let ResNet-M be its SubResNet (Definition \ref{def:subresnet}), trained without the stage cost but the same weight decay.
	Let $\bar{\mathcal{L}}=\lossfunc(\hat{\mathbf{y}}^M)$ be the loss function of the output of ResNet-M.
	Then, the training objective for the optimal parameters $\mathbf{w}^N$ ResNet-$N$ satisfies the upper bound
	\begin{equation*}
		J_N(\mathbf{w}^N) = \sum_{k = 0}^{N-1} \left[ \gamma \lossfunc(\hat{\mathbf{y}}^N_k)  +
		\frac{\lambda}{2} \lVert\mathbf{w}^N_{\resterm,k}\rVert^2 \right] + 
		\lossfunc(\hat{\mathbf{y}}^N) 
		\leq 
		\bar{J}_N,
	\end{equation*}
	with $\bar{J}_N = \sum_{k = 0}^{M-1} \left[\gamma \lossfunc(\hat{\mathbf{y}}^M_k) + \frac{\lambda}{2} \lVert\mathbf{w}^M_{\resterm,k}\rVert^2 \right]
	+ (1 + \gamma(N - M)) \bar{\mathcal{L}}$ based on the objective of ResNet-$M$.
\end{T1}
\begin{proof}
	Let $\hat{\mathbf{y}}^M_k$ denote the trajectory of the intermediate outputs of the ResNet-M.
	Then by Lemma \ref{lemma:identity_mapping_subresnets}, the feasible parameters $\bar{\mathbf{w}}^N$ of for ResNet-N, with $\bar{\mathbf{w}}_k^N = \mathbf{w}_k^M$ $k = 0, ..., M-1$ and $\bar{\mathbf{w}}^N_{\resterm, k} = 0$ for $k = N, ..., M-1$, produce the trajectory of intermediate outputs 
	\begin{equation*}
		\bar{\mathbf{y}}^N_k = \begin{cases}
			\hat{\mathbf{y}}^M_k & \text{if }k < M \\
			\hat{\mathbf{y}}^M_{M} & \text{if } k \geq M,
		\end{cases}
	\end{equation*}
	for all $k=0, ..., N$.
	The feasible parameters $\bar{\mathbf{w}}^N$ induce in the objective with stage cost loss
	\begin{equation*}
		\bar{J}_N = J_N(\bar{\mathbf{w}}^N)
		=\sum_{k = 0}^{M-1} \left[\gamma \lossfunc(\hat{\mathbf{y}}^M_k)  + 
		\frac{\lambda}{2} \lVert\bar{\mathbf{w}}^M_{\resterm,k} \rVert^2  \right]
		+ ( \gamma(N - M) + 1) \bar{\mathcal{L}}.
	\end{equation*}
	Which is an upper bound for the optimal intermediate outputs $\hat{\mathbf{y}}^N_k$ and parameters $\mathbf{w}^N$ obtained from training with the stage cost loss
	\begin{equation*}
		J_N(\mathbf{w}^N) = \sum_{k = 0}^{N-1} \left[\gamma \lossfunc(\hat{\mathbf{y}}^N_k)  +
		\frac{\lambda}{2} \lVert \mathbf{w}^N_{\resterm,k}\rVert^2 \right]
		+ \lossfunc(\hat{\mathbf{y}}^N_{N}) 
		\leq 	\bar{J}_N.
	\end{equation*}
\end{proof}

The second bound is an extension of Theorem \ref{thrm:deep_resnet_asymptotics_wd} for the training with stage cost loss in the absence of weight decay.
\begin{theorem}[ResNet asymptotics without weight decay]
	\label{lemma:deep_resnet_asymptotics}
	Consider a ResNet-N trained with stage cost, $\gamma > 0$, and $N$ residual blocks and its SubResNet architecture (Definition \ref{def:subresnet}) ResNet-M trained without stage cost.
	Let $\bar{\mathcal{L}}=\lossfunc(\hat{\mathbf{y}}^M)$ be the loss function of the output of ResNet-M.
	Then, the average loss along the intermediate outputs $\mathcal{L}_\mathrm{avg}$ of ResNet-N satisfies
	\begin{equation}
		\mathcal{L}_\mathrm{avg}= \frac{1}{N+1} \sum_{k = 0}^{N} \lossfunc(\hat{\mathbf{y}}^N_k) \leq \bar{ \mathcal{L}}_\mathrm{avg} = \bar{\mathcal{L}} + \frac{C}{N+1},
		\label{eq:deep_resnet_asymptotics}
	\end{equation}
	for $C = \sum_{k = 0}^{M-1} \lossfunc(\hat{\mathbf{y}}^M_k)  + \frac{1 - \gamma -\gamma M}{\gamma} \bar{\mathcal{L}} $ relating to the performance of the shallower network ResNet-M.
\end{theorem}
\begin{proof}
	Let $\hat{\mathbf{y}}^M_k$ denote the trajectory of intermediate outputs of ResNet-M. 
	Then by Lemma \ref{lemma:identity_mapping_subresnets}, choosing the parameters of for ResNet-N $\bar{\mathbf{w}}^N$ as $\bar{\mathbf{w}}_k^N = \mathbf{w}_k^M$ $k = 0, ..., M-1$ and $\bar{\mathbf{w}}^N_{\resterm, k} = 0$ for $k = M, ..., N-1$ produce the trajectory of intermediate outputs 
	\begin{equation*}
		\bar{\mathbf{y}}^N_k = \begin{cases}
			\hat{\mathbf{y}}^M_k & \text{if }k < M \\
			\hat{\mathbf{y}}^M_{N} & \text{if }k \geq N,
		\end{cases}
	\end{equation*}
	for $k = 0,..., N-1$.
	This feasible trajectory of parameters $\bar{\mathbf{w}}^N$ and intermediate outputs $\bar{\mathbf{y}}^N$ induces in the objective with stage cost
	\begin{equation*}
		J_N(\bar{\mathbf{w}}^{N})= \sum_{k = 0}^{N-1} \gamma \lossfunc(\bar{\mathbf{y}}^N_k)  + \lossfunc(\bar{\mathbf{y}}^N_{N}) 
		=\sum_{k = 0}^{M-1} \gamma \lossfunc(\hat{\mathbf{y}}^M_k)  + ( \gamma(N - M) + 1) \bar{\mathcal{L}}.
	\end{equation*}
	This is an upper bound to training objective of the trained ResNet-N trained with the stage cost loss with the optimal parameters $\mathbf{w}^N$ and intermediate outputs $\hat{\mathbf{y}}^M$
	\begin{equation*}
		J_N(\mathbf{w}^N) = \sum_{k = 0}^{N-1} \gamma \lossfunc(\hat{\mathbf{y}}^N_k)  + \lossfunc(\hat{\mathbf{y}}^N_{N}) 
		\leq 	J_N(\bar{\mathbf{w}}^N) = \sum_{k = 0}^{M-1} \gamma \lossfunc(\hat{\mathbf{y}}^M_k)  + (1 + \gamma(N - M)) \bar{\mathcal{L}}.
	\end{equation*}
	Using that $\gamma \leq 1$, we obtain
	\begin{equation*}
		\sum_{k = 0}^{N} \gamma \lossfunc(\hat{\mathbf{y}}^N_k) \leq \sum_{k = 0}^{M-1} \gamma \lossfunc(\hat{\mathbf{y}}^M_k)  + (1 + \gamma(N - M)) \bar{\mathcal{L}}.
	\end{equation*}
	With $\bar{\mathcal{L}} = \gamma \bar{\mathcal{L}} + (1-\gamma)\bar{\mathcal{L}}$ and by rearranging we obtain
	\begin{equation}
		\sum_{k = 0}^{N} \gamma \lossfunc(\hat{\mathbf{y}}^N_k) \leq \sum_{k = 0}^{M-1} \gamma \lossfunc(\hat{\mathbf{y}}^N_k)  + (1-\gamma - \gamma M) \bar{\mathcal{L}} + \gamma (N +1) \bar{\mathcal{L}}. 
		\label{eq:lemma_avg_loss_full_resnet}
	\end{equation}
	Dividing by $\gamma (N+1)$ gives \eqref{eq:deep_resnet_asymptotics} with $C = \sum_{k = 0}^{M-1} \lossfunc(\hat{\mathbf{y}}^M_k)  + \frac{1 - \gamma -\gamma M}{\gamma} \bar{\mathcal{L}} $.
\end{proof}
This implies that asymptotically, for $N \to \infty$, the average loss of of the ResNet trained with the stage cost loss will be equal to that of the shallower SubResNet.

However, we emphasize that these bounds focus solely on the training dynamics, and no guarantees are provided regarding validation performance.
Our theoretical results require that the empirical loss is optimized to global optimality.
Moreover, if a loss function that attains its minimum at zero is used, e.g., the L2 loss, and assuming that there exists a SubResNet of depth $M$ such that $\bar{\mathcal{L}} =0$, then this bound tightens to $\mathcal{L}_\mathrm{avg} \leq \frac{C}{N+1}$.

The results of the bounds for CIFAR-10 are shown in Table~\ref{tab:results_loss_bounds_cifar_10_standard} for the standard ResNet and in Table~\ref{tab:results_loss_bounds_cifar_10_homogeneous} for the homogeneous ResNet.
For the homogeneous ResNet, we find Bound 1 on the training objective with weight decay (Theorem~\ref{thrm:deep_resnet_asymptotics_wd}) to be moderately tight.
The bound obtained from the SubResNet of depth 6 is closest, since it is the shallowest network that achieves a sufficiently low training loss.
The SubResNet of depth 0, which cannot fit the dataset properly, leads to a looser depth bound due to its larger value of $\bar{\mathcal{L}}$.
For deeper SubResNets, the summation of the higher loss function values of the intermediate outputs leads to an overly conservative bound.
We evaluate the Bound 2 from Theorem~\ref{lemma:deep_resnet_asymptotics} for the training without weight decay for the trajectories obtained by training with small weight decay of $\lambda=10^{-4}$.
The results on the average loss function are less tight due to the higher loss values for the outputs of the first residual blocks of the ResNet trained by standard training.
In contrast to Bound 1, these are not weighted by the stage cost loss weight $\gamma$.

In general, we find that for the standard ResNet, our bounds are less tight. One reason is this difference between the representational abilities of the full ResNet and the shallower SubResNets, which only contain residual blocks with 16 and 32 convolution filters.

\begin{table}[H]
	\centering
	\caption{Results of loss bounds for the standard ResNet trained on CIFAR-10.}
	\label{tab:results_loss_bounds_cifar_10_standard}
	\begin{tabular}{lccccccc}
\toprule
SubResNet &  $\bar{\mathcal{L}}$ & \multicolumn{3}{c}{\textbf{Bound 1} (Theorem \ref{thrm:deep_resnet_asymptotics_wd})} & \multicolumn{3}{c}{\textbf{Bound 2} (Theorem \ref{lemma:deep_resnet_asymptotics})} \\
depth $M$ & & $J_M(w^M)$ & $\bar{J}_M$ & $\bar{J}_M - J_M(w^M)$ & $\mathcal{L}_\mathrm{avg}$ & $\bar{\mathcal{L}}_\mathrm{avg}$ & $\bar{\mathcal{L}}_\mathrm{avg} - \mathcal{L}_\mathrm{avg}$ \\
\midrule
0 & 1.326 & 0.570 & 2.758 & 2.187 & 0.340 & 2.507 & 2.167 \\
6 & 0.327 & 0.570 & 1.173 & 0.603 & 0.340 & 21.238 & 20.898 \\
12 & 0.252 & 0.570 & 1.186 & 0.616 & 0.340 & 28.063 & 27.723 \\
18 & 0.228 & 0.570 & 1.274 & 0.704 & 0.340 & 34.451 & 34.112 \\
24 & 0.072 & 0.570 & 1.549 & 0.979 & 0.340 & 58.629 & 58.289 \\
30 & 0.006 & 0.570 & 2.128 & 1.558 & 0.340 & 90.931 & 90.591 \\
36 & 0.004 & 0.570 & 2.409 & 1.839 & 0.340 & 103.758 & 103.418 \\
42 & 0.002 & 0.570 & 2.189 & 1.618 & 0.340 & 93.293 & 92.953 \\
48 & 0.001 & 0.570 & 2.253 & 1.682 & 0.340 & 96.334 & 95.994 \\
54 & 0.002 & 0.570 & 2.232 & 1.662 & 0.340 & 94.822 & 94.483 \\
\bottomrule
\end{tabular}

\end{table}

\begin{table}[H]
	\centering
	\caption{Results of loss bounds for the homogeneous ResNet trained on CIFAR-10.}
	\label{tab:results_loss_bounds_cifar_10_homogeneous}
	\begin{tabular}{lccccccc}
\toprule
\textbf{SubResNet} & $\bar{\mathcal{L}}$ & \multicolumn{3}{c}{\textbf{Bound 1} Theorem \ref{thrm:deep_resnet_asymptotics_wd})} & \multicolumn{3}{c}{\textbf{Bound 2} (Theorem \ref{lemma:deep_resnet_asymptotics})} \\
\textbf{depth} $M$ &  & $J_M(w^M)$ & $\bar{J}_M$ & $\bar{J}_M - J_M(w^M)$ & $\mathcal{L}_\mathrm{avg}$ & $\bar{\mathcal{L}}_\mathrm{avg}$ & $\bar{\mathcal{L}}_\mathrm{avg} - \mathcal{L}_\mathrm{avg}$ \\
\midrule
0 & 1.188 & 0.245 & 2.470 & 2.225 & 0.087 & 2.246 & 2.159 \\
6 & 0.003 & 0.245 & 0.421 & 0.176 & 0.087 & 13.275 & 13.188 \\
12 & 0.001 & 0.245 & 0.537 & 0.292 & 0.087 & 18.518 & 18.431 \\
18 & 0.001 & 0.245 & 0.574 & 0.330 & 0.087 & 20.260 & 20.173 \\
24 & 0.005 & 0.245 & 0.728 & 0.483 & 0.087 & 25.312 & 25.225 \\
30 & 0.001 & 0.245 & 0.761 & 0.516 & 0.087 & 28.229 & 28.142 \\
36 & 0.001 & 0.245 & 0.813 & 0.568 & 0.087 & 30.843 & 30.756 \\
42 & 0.001 & 0.245 & 0.801 & 0.556 & 0.087 & 30.287 & 30.200 \\
48 & 0.004 & 0.245 & 0.925 & 0.680 & 0.087 & 34.440 & 34.353 \\
54 & 0.001 & 0.245 & 0.942 & 0.697 & 0.087 & 36.772 & 36.685 \\
\bottomrule
\end{tabular}

\end{table}

\end{document}